%% file: main.tex
\pdfoutput=1
\newif\iflongversion
\longversiontrue
\documentclass{article}
\usepackage{nips15submit_e,times}
\usepackage{preamble}
\usepackage{hyperref}

\newcommand{\Topk}{Top-\texorpdfstring{$k$}{k}}
\newcommand{\Sm}{\Delta}
\newcommand{\Smk}{\Delta_k}
\newcommand{\SUmk}{\tilde{\Delta}_k}
\newcommand{\compatible}[1]{$#1$-compatible}

\newcommand{\MethodTopK}[1]{${\rm Top\mhyphen #1}$}
\newcommand{\MethodSvmOva}{${\rm SVM}^{\rm OVA}$}
\newcommand{\MethodTopPushOva}{${\rm TopPush}^{\rm OVA}$}
\newcommand{\MethodTopPushMulti}{${\rm TopPush}$}

\newcommand{\MethodSvmRecKOva}[1]{${\rm Recall@#1}^{\rm OVA}$}
\newcommand{\MethodSvmPrecKMulti}[1]{${\rm Prec@#1}$}
\newcommand{\MethodSvmRecKMulti}[1]{${\rm Recall@#1}$}
\newcommand{\MethodSvmPerf}{${\rm SVM}^{\rm Perf}$}
\newcommand{\MethodSvmTopK}[1]{${\rm top}\mhyphen {#1}~{\rm SVM_\alpha}$}
\newcommand{\MethodUsuTopK}[1]{${\rm top}\mhyphen {#1}~{\rm SVM_\beta}$}
\newcommand{\MethodWsabie}[2]{${\rm W{\scriptstyle ++}, {\scriptstyle #2/#1}}$}

\title{Top-k Multiclass SVM}

\author{
Maksim Lapin,$^1$
Matthias Hein$^2$
and
Bernt Schiele$^1$ \\
$^1$Max Planck Institute for Informatics, Saarbr\"ucken, Germany \\
$^2$Saarland University, Saarbr\"ucken, Germany
}

\nipsfinalcopy %

\begin{document}

\maketitle

\input{sections/abstract}

\section{Introduction}
\label{sec:introduction}
\input{sections/introduction}

\section{\Topk\ Loss in Multiclass Classification}
\label{sec:topkloss}
\input{sections/topkloss}

\section{Optimization Framework}
\label{sec:optimization}
\input{sections/optimization}

\section{Efficient Projection onto the \Topk\ Simplex}
\label{sec:projection}
\input{sections/projection}

\iflongversion
\section{Optimization of \Topk\ Usunier Loss}
\label{sec:usunier}
\input{sections/usunier}
\fi

\section{Experimental Results}
\label{sec:experiments}
\input{sections/experiments}

\section{Conclusion}
\label{sec:conclusion}
\input{sections/conclusion}

{\small
\bibliographystyle{ieee}
\bibliography{main}
}

\end{document}

%% file: sections/abstract.tex
\vspace*{-1em}
\begin{abstract}
Class ambiguity is typical in image
classification problems with a large number of classes.
When classes are difficult to discriminate,
it makes sense to allow $k$ guesses and evaluate classifiers
based on the top-$k$ error instead of the standard zero-one loss.
We propose top-$k$ multiclass SVM
as a direct method to optimize for top-$k$ performance.
Our generalization of the well-known multiclass SVM is based on
a tight convex upper bound of the top-$k$ error. 
We propose a fast optimization scheme
based on an efficient projection onto the
top-$k$ simplex, which is of its own interest.
Experiments on five datasets
show consistent improvements in top-$k$ accuracy
compared to various baselines.
\end{abstract}

%% file: sections/introduction.tex
\input{figures/figure-teaser-2rows}

As the number of classes increases,
two important issues emerge:
class overlap and multi-label nature of examples \cite{Gupta2014}.
This phenomenon asks for adjustments of both
the evaluation metrics as well as the loss functions employed.
When a predictor is allowed $k$ guesses 
and is not penalized for $k-1$ mistakes,
such an evaluation measure is known as top-$k$ error.
We argue that this is an important metric that will inevitably receive
more attention in the future as the illustration in Figure~\ref{fig:teaser} 
indicates.

How obvious is it
that each row of Figure~\ref{fig:teaser} 
shows examples
of \emph{different} classes?
Can we imagine a human to predict correctly on the first attempt?
Does it even make sense to penalize a learning system
for such ``mistakes''? While the problem of
class ambiguity is apparent in computer vision, 
similar problems arise in other domains when the number of classes
becomes large.

We propose top-$k$ multiclass SVM
as a generalization of the well-known multiclass SVM \cite{crammer2001algorithmic}.
It is based on a tight convex upper bound of the top-$k$ zero-one loss
which we call \textbf{top-$k$ hinge loss}. While it turns out to be similar
to a top-$k$ version of the ranking based loss proposed by \cite{usunier2009ranking},
we show that the top-$k$ hinge loss is a lower bound on their version
and is thus a tighter bound on the top-$k$ zero-one loss.
We propose an efficient implementation based on
stochastic dual coordinate ascent (SDCA)
\cite{shalev2014accelerated}. 
A key ingredient in the optimization is the (biased) projection onto the top-$k$ simplex.
This projection turns out to be a tricky generalization of the continuous
quadratic knapsack problem, respectively the projection onto the standard simplex.
The proposed algorithm for solving it has complexity
$O(m \log m)$ for $x \in \Rb^m$.
Our implementation of the top-$k$ multiclass SVM scales to
large datasets like Places 205 with about $2.5$ million examples and $205$ classes
\cite{zhou2014learning}.
Finally, extensive experiments on several challenging computer vision
problems show that top-$k$ multiclass SVM consistently improves
in top-$k$ error over the multiclass SVM
(equivalent to our top-$1$ multiclass SVM),
one-vs-all SVM and other methods based on
different ranking losses \cite{joachims2005support,li2014top}.

%% file: figures/figure-teaser-2rows.tex
\begin{wrapfigure}[10]{r}{0.5\textwidth}%
\vspace*{-3em}
\begin{minipage}[c]{.95\linewidth}%
\setlength\fboxsep{0pt}%
\setlength\fboxrule{0.5pt}
\fbox{\includegraphics[width=0.353\columnwidth]{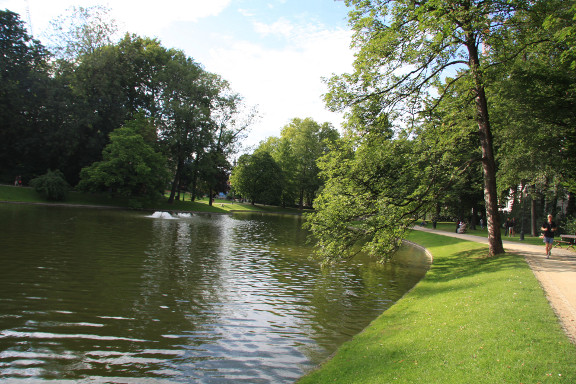}}%
\fbox{\includegraphics[width=0.280\columnwidth]{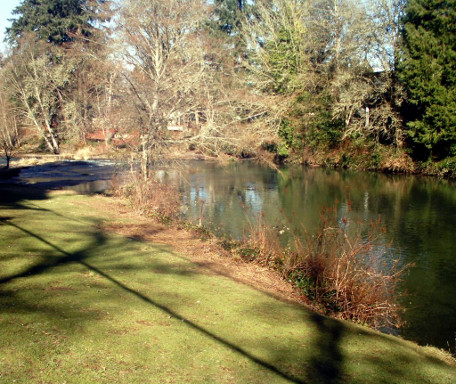}}%
\fbox{\includegraphics[width=0.353\columnwidth]{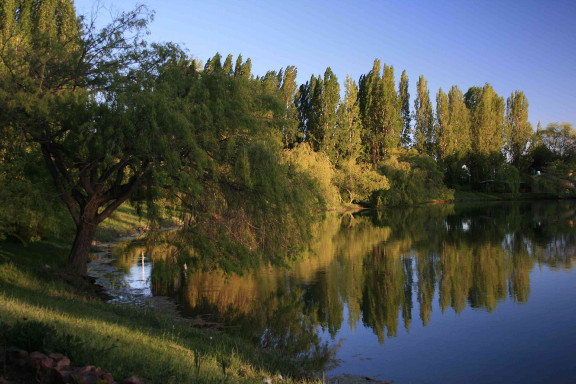}}\\
\fbox{\includegraphics[width=0.355\columnwidth]{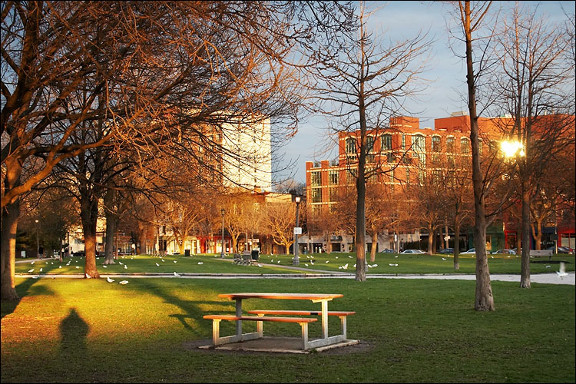}}%
\fbox{\includegraphics[width=0.315\columnwidth]{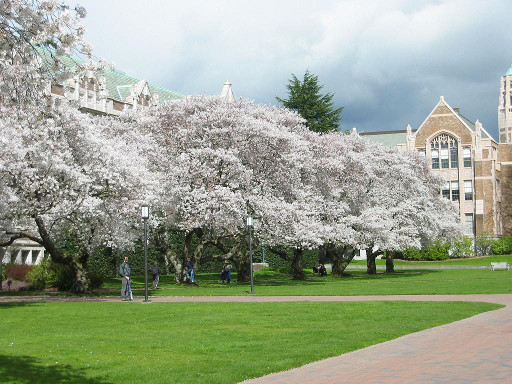}}%
\fbox{\includegraphics[width=0.315\columnwidth]{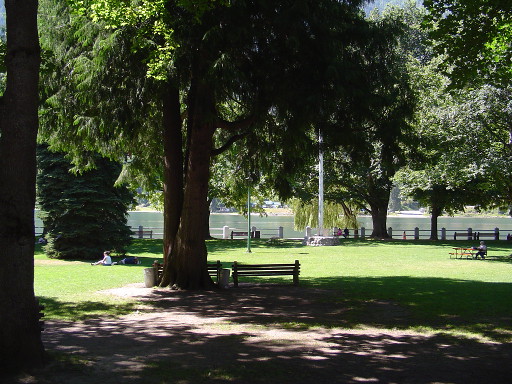}}
\end{minipage}
\caption{%
Images from SUN 397~\cite{xiao2010sun} illustrating class ambiguity.
{\bfseries Top:} (left to right) Park, River, Pond.
{\bfseries Bottom:} Park, Campus, Picnic area.
}\label{fig:teaser}%
\end{wrapfigure}

%% file: sections/topkloss.tex
In multiclass classification, one is given a set
$
S = \{ (x_i, y_i) \given i = 1, \ldots, n \}
$
of $n$ training examples $x_i \in \Xc$
along with the corresponding labels $y_i \in \Yc$.
Let $\Xc = \Rb^d$ be the feature space and $\Yc = \{ 1, \ldots, m \}$
the set of labels. 
The task is to learn a set of $m$ linear predictors $w_y \in \Rb^d$
such that the risk of the classifier
$\argmax_{y \in \Yc} \inner{w_y, x}$
is minimized for a given loss function,
which is usually chosen to be a convex upper bound of the zero-one loss.
The generalization to nonlinear predictors using kernels is discussed below.

The classification problem becomes extremely challenging
in the presence of a large number of ambiguous classes.
It is natural in that case to extend the evaluation protocol
to allow $k$ guesses, which leads to the popular
top-$k$ error and top-$k$ accuracy performance measures.
Formally, we consider a ranking of labels induced by
the prediction scores $\inner{w_y, x}$.
Let the bracket $[\cdot]$ denote a permutation
of labels such that $[j]$ is the index of the $j$-th largest score, \ie
\begin{align*}
\inner{w_{[1]},x} \geq \inner{w_{[2]},x} \geq \ldots \geq \inner{w_{[m]},x}.
\end{align*}
The top-$k$ zero-one loss $\err_k$ is defined as
\begin{align*}%
\err_k(f(x), y) %
=\ind{\innern{w_{[k]},x} > \innern{w_{y},x}} ,
\end{align*}
where
$f(x) = \left( \inner{w_1, x}, \ldots, \inner{w_m, x} \right)^{\top}$
and
$\ind{P} = 1$ if $P$ is true and $0$ otherwise.
Note that the standard zero-one loss is recovered when $k=1$,
and $\err_k(f(x), y)$ is always $0$ for $k = m$.
Therefore, we are interested in the regime $1 \leq k < m$.

\subsection{Multiclass Support Vector Machine}
\label{sec:msvm}

In this section we review the multiclass SVM
of Crammer and Singer \cite{crammer2001algorithmic}
which will be extended to the top-$k$ multiclass SVM in the following.
We mainly follow the notation of \cite{shalev2014accelerated}.

Given a training pair $(x_i,y_i)$,
the multiclass SVM loss on example $x_i$ is defined as
\begin{align}\label{eq:msvm-loss}
\max_{y \in \Yc} \{
\ind{y \neq y_i} + \inner{w_y, x_i} - \inner{w_{y_i}, x_i} \} .
\end{align}

Since our optimization scheme is based on Fenchel duality,
we also require a convex conjugate of the primal loss function
(\ref{eq:msvm-loss}).
Let $c \bydef \ones - e_{y_i}$,
where $\ones$ is the all ones vector
and $e_j$ is the $j$-th standard basis vector in $\Rb^m$,
let $a \in \Rb^m$ be defined componentwise as
$a_j \bydef \inner{w_j, x_i} - \inner{w_{y_i}, x_i}$,
and let
\begin{align*}
\Sm \bydef \{ x \in \Rb^m \given \inner{\ones, x} \leq 1,\; 0 \leq x_i,
\; i = 1, \ldots, m \}.
\end{align*}

\begin{proposition}[\cite{shalev2014accelerated}, \S~5.1]
A primal-conjugate pair for the multiclass SVM loss (\ref{eq:msvm-loss}) is
\begin{align}\label{eq:msvm-pair}
\phi(a) &= \max \{0, (a + c)_{[1]} \}, &
\phi^*(b) &=
\begin{cases}
- \inner{c, b} & \text{if } b \in \Sm , \\
+\infty & \text{otherwise} .
\end{cases}
\end{align}
\end{proposition}

Note that thresholding with $0$ in $\phi(a)$ is actually redundant
as $(a + c)_{[1]} \geq (a + c)_{y_i} = 0$ and is only given
to enhance similarity to the top-$k$ version defined later.

\subsection{\Topk\ Support Vector Machine}
\label{sec:topk-svm}

The main motivation for the top-$k$ loss
is to relax the penalty for making an error
in the top-$k$ predictions.
Looking at $\phi$ in (\ref{eq:msvm-pair}),
a direct extension to the top-$k$ setting would be
a function
\begin{align*}
\psi_k(a) = \max \{0, (a + c)_{[k]} \} ,
\end{align*}
which incurs a loss iff $(a + c)_{[k]} > 0$.
Since the ground truth score $(a + c)_{[y_i]} = 0$,
we conclude that
\begin{align*}
\psi_k(a) > 0 \; \Longleftrightarrow \;
\inner{w_{[1]}, x_i} \geq \ldots
\geq \inner{w_{[k]}, x_i} > \inner{w_{y_i}, x_i} - 1 ,
\end{align*}
which directly corresponds to the top-$k$ zero-one loss $\err_k$
with margin $1$.
\iflongversion\else

\fi
Note that the function $\psi_k$ ignores the values
of the first $(k-1)$ scores, which could be quite large
if there are highly similar classes.
That would be fine in this model
as long as the correct prediction is within the first $k$ guesses.
However, the function $\psi_k$ is unfortunately nonconvex
since the function $f_k(x) = x_{[k]}$
returning the $k$-th largest coordinate
is nonconvex for $k \geq 2$.
Therefore, finding a globally optimal solution is computationally
intractable.

Instead, we propose the following convex upper bound on $\psi_k$,
which we call the \textbf{top-$k$ hinge loss},
\begin{align}\label{eq:topk-loss}
\phi_k(a) = \max \Big\{0, \frac{1}{k} \sum_{j=1}^k (a + c)_{[j]} \Big\} ,
\end{align}
where the sum of the $k$ largest components is known to be convex
\cite{boyd2004convex}. We have that
\begin{align*}
\psi_k(a) \leq \phi_k(a) \leq \phi_1(a) = \phi(a) ,
\end{align*}
for any $k \geq 1$ and $a \in \Rb^m$.
Moreover, $\phi_k(a) < \phi(a)$ unless all $k$ largest scores are the same.
This extra slack can be used to increase the margin between the current
and the $(m-k)$ remaining least similar classes,
which should then lead to an improvement in the top-$k$ metric.

\subsubsection{\Topk\ Simplex and Convex Conjugate of the \Topk\ Hinge Loss}
\label{sec:topk-svm-conjugate}

In this section we derive the conjugate of the proposed loss
(\ref{eq:topk-loss}).
We begin with a well known result that is used later in the proof.
All proofs can be found in the supplement.
Let $[a]_+ = \max\{0, a\}$.

\begin{lemma}[\cite{ogryczak2003minimizing}, Lemma~1]\label{lem:sumk}
$ \sum_{j=1}^k h_{[j]} = \min_t \big\{ kt + \sum_{j=1}^m [h_j - t]_+ \big\}.$
\end{lemma}
\iflongversion
\begin{proof}
For a $t_0 \in [h_{[k+1]}, h_{[k]}]$, we have
\begin{align*}
\min_t \big\{ kt + \sum_{j=1}^m [h_j - t]_+ \big\}
\leq kt_0 + \sum_{j=1}^m [h_j - t_0]_+ =
kt_0 + \sum_{j=1}^k \left(h_{[j]} - t_0 \right) =
\sum_{i=1}^k h_{[i]} .
\end{align*}
On the other hand, for any $t \in \Rb$, we get
\begin{align*}
\sum_{j=1}^k h_{[j]}
= kt + \sum_{j=1}^k \left( h_{[j]} - t \right)
\leq kt + \sum_{j=1}^k \left[ h_{[j]} - t \right]_+
\leq kt + \sum_{j=1}^m \left[ h_{j} - t \right]_+ .
\end{align*}
\end{proof}
\fi

\input{figures/plot-simplex}
We also define a set $\Smk$ which arises naturally
as the effective domain\footnote{
A convex function $f: X \rightarrow \Rb \cup \{\pm \infty \}$
has an \emph{effective domain}
$\dom f = \{x \in X \given f(x) < +\infty \}$.}
of the conjugate of (\ref{eq:topk-loss}).
By analogy, we call it the top-$k$ simplex
as for $k=1$ it reduces to the standard simplex
with the inequality constraint (\ie $0 \in \Smk$).
Let $[m] \bydef 1, \ldots, m$.

\begin{definition}\label{def:topk-simplex}
The \emph{top-$k$ simplex} is a convex polytope defined as
\begin{align*}
\Smk(r) \bydef
\left\{ x \given \inner{\ones, x} \leq r, \;
0 \leq x_i \leq \frac{1}{k} \inner{\ones, x},
\; i \in [m] \right\} ,
\end{align*}
where $r \geq 0$ is the bound on the sum $\inner{\ones, x}$.
We let $\Smk \bydef \Smk(1)$.
\end{definition}
The crucial difference to the standard simplex is the upper bound on $x_i$'s,
which limits their maximal contribution to the total sum $\inner{\ones, x}$.
See Figure~\ref{fig:plot-simplex} for an illustration.

The first technical contribution of this work is as follows.
\begin{proposition}\label{prop:topk-pairs}
A primal-conjugate pair for the top-$k$ hinge loss (\ref{eq:topk-loss}) is
given as follows:
\begin{align}\label{eq:topk-pair}
\phi_k(a) &=
\max \Big\{0, \frac{1}{k} \sum_{j=1}^k (a + c)_{[j]} \Big\} , &
\phi_k^*(b) &=
\begin{cases}
- \inner{c, b} & \text{if } b \in \Smk , \\
+\infty & \text{otherwise} .
\end{cases}
\end{align}
Moreover,
$\phi_k(a) = \max \{ \inner{a + c, \lambda} \given \lambda \in \Smk \}$.
\end{proposition}
\iflongversion
\begin{proof}
We use Lemma~\ref{lem:sumk} to write
\begin{align*}
\phi_k(a) = \min \big\{ s \, \given \,
s \geq t + \frac{1}{k} \sum_{j=1}^m \xi_j , \;
s \geq 0 , \;
\xi_j \geq a_j + c_j - t , \;
\xi_j \geq 0 \big\} .
\end{align*}
The Lagrangian is given as
\begin{align*}
\Lc(s,t,\xi,\alpha,\beta,\lambda,\mu) =
s + \alpha \big( t + \frac{1}{k} \sum_{j=1}^m \xi_j - s \big)
- \beta s 
+ \sum_{j=1}^m \lambda_j \left( a_j + c_j - t - \xi_j \right)
- \sum_{j=1}^m \mu_j \xi_j .
\end{align*}
Minimizing over $(s, t, \xi)$, we get
$\alpha + \beta = 1$,
$\alpha = \sum_{j=1}^m \lambda_j$,
$\lambda_j + \mu_j = \frac{1}{k} \alpha$.
As $\beta \geq 0$ and $\mu_j \geq 0$, it follows that
$\inner{\ones, \lambda} \leq 1$ and
$0 \leq \lambda_j \leq \frac{1}{k} \inner{\ones, \lambda}$.
Since the duality gap is zero, we get
\begin{align*}
\phi_k(a) = \max \{ \inner{a + c, \lambda} \given \lambda \in \Smk \} .
\end{align*}
The conjugate $\phi_k^*(b)$ can now be computed as
\begin{align*}
\max_a \{ \inner{a,b} - \phi_k(a) \} =
\max_a \min_{\lambda \in \Smk} \{ \inner{a,b} - \inner{a + c, \lambda} \} =
\min_{\lambda \in \Smk} \{ - \inner{c, \lambda}
+ \max_a \inner{a,b - \lambda} \} .
\end{align*}
Since $\max_a \inner{a,b - \lambda} = \infty$ unless
$b = \lambda$, we get the formula for $\phi_k^*(b)$ as in
(\ref{eq:topk-pair}).

\end{proof}
\fi

Therefore, we see that the proposed formulation (\ref{eq:topk-loss})
naturally extends the multiclass SVM of Crammer and Singer
\cite{crammer2001algorithmic}, which is recovered when $k=1$.
We have also obtained an interesting extension
(or rather contraction, since $\Smk \subset \Sm$)
of the standard simplex.

\subsection{Relation of the \Topk\ Hinge Loss to Ranking Based Losses}
\label{sec:comparison}

Usunier \etal \cite{usunier2009ranking}
have recently formulated a very general family of convex losses
for ranking and multiclass classification.
In their framework, the hinge loss on example $x_i$
can be written as
\[
L_{\beta}( a )=\sum_{y=1}^m \beta_y \max\{0,(a+c)_{[y]}\},
\]
where $\beta_1 \geq \ldots \geq \beta_m \geq 0$
is a non-increasing sequence of non-negative numbers
which act as weights for the ordered losses.
\iflongversion\else

\fi
The relation to the top-$k$ hinge loss becomes apparent if we choose
$\beta_j = \frac{1}{k}$ if $j \leq k$, and $0$ otherwise.
In that case, we obtain another version of the top-$k$ hinge loss
\begin{align}\label{eq:topk-usu}
\tilde{\phi}_k\big( a \big)=
\frac{1}{k}\sum_{j=1}^k \max\{0, (a+c)_{[j]}\}.
\end{align}
It is straightforward to check that
\begin{align*}
\psi_k(a) \leq \phi_k(a) \leq \tilde{\phi}_k(a) \leq \phi_1(a)
= \tilde{\phi}_1(a) = \phi(a).
\end{align*}
The bound $\phi_k(a) \leq \tilde{\phi}_k(a)$ holds with equality if $(a+c)_{[1]}\leq 0$ or $(a+c)_{[k]}\geq 0$.
Otherwise, there is a gap and our top-$k$ loss
is a strictly better upper bound on the actual top-$k$ zero-one loss.
While \cite{usunier2009ranking} employed LaRank~\cite{bordes2007solving}
and \cite{Gupta2014}, \cite{Weston2011} optimized an approximation
of $L_{\beta}( a )$, %
we show in
\iflongversion
\S~\ref{sec:usunier}
\fi
how the loss function (\ref{eq:topk-usu}) can be optimized exactly and efficiently
within the Prox-SDCA framework.

\textbf{Multiclass to binary reduction.}
It is also possible to compare directly to ranking based methods
that solve a binary problem using the following reduction.
We employ it in our experiments to evaluate
the ranking based methods \MethodSvmPerf\ \cite{joachims2005support}
and
\MethodTopPushMulti\ \cite{li2014top}.
The trick is to augment the training set by
embedding each $x_i \in \Rb^d$ into $\Rb^{md}$ using
a feature map $\Phi_y$ for each $y \in \Yc$.
The mapping $\Phi_y$ places $x_i$ at the $y$-th position
in $\Rb^{md}$ and puts zeros everywhere else.
The example $\Phi_{y_i}(x_i)$ is labeled $+1$ and all
$\Phi_{y}(x_i)$ for $y \neq y_i$ are labeled $-1$.
Therefore, we have a new training set with $mn$ examples
and $md$ dimensional (sparse) features.
Moreover, %
$\inner{w, \Phi_y(x_i)} = \inner{w_y,x_i}$
which establishes the relation to the original multiclass problem.

Another approach to general performance measures is given in \cite{joachims2005support}.
It turns out that using the above reduction,
one can show that under certain constraints on the classifier, the recall@$k$
is equivalent to the top-$k$ error.
A convex upper bound on recall@$k$ is then optimized in \cite{joachims2005support}
via structured SVM.
As their convex upper bound on the recall@$k$
is not decomposable in an instance based loss,
it is not directly comparable to our loss.
While being theoretically very elegant,
the approach of \cite{joachims2005support} does not scale to very large datasets.

%% file: figures/plot-simplex.tex
\begin{wrapfigure}[15]{l}{0.33\textwidth}%
\includegraphics[width=.95\linewidth]{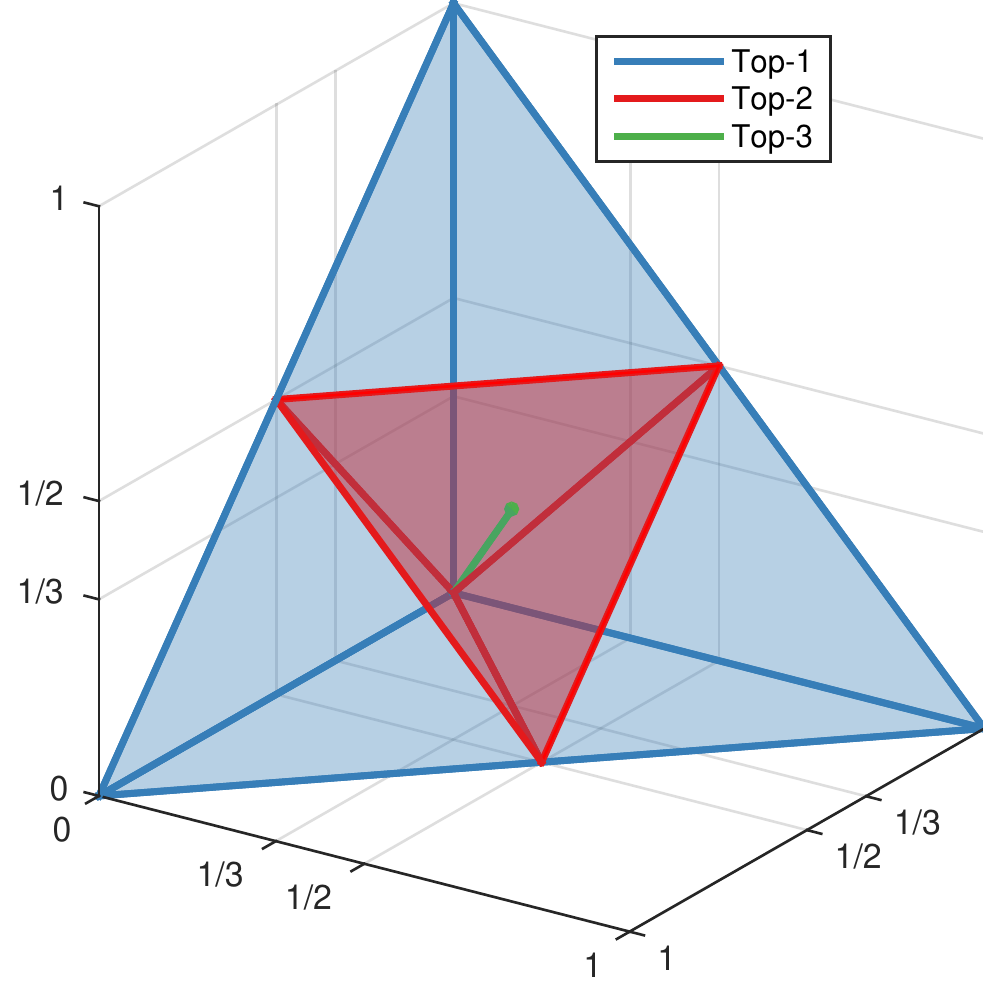}%
\caption{%
Top-$k$ simplex $\Smk(1)$ for $m=3$.
Unlike the standard simplex,
it has $\binom{m}{k} + 1$ vertices. %
}\label{fig:plot-simplex}%
\end{wrapfigure}

%% file: sections/optimization.tex
We begin with a general $\ell_2$-regularized
multiclass classification problem,
where for notational convenience we keep the loss function unspecified.
The multiclass SVM or the top-$k$ multiclass SVM %
are obtained by plugging in the corresponding loss function from
\S~\ref{sec:topkloss}.

\subsection{Fenchel Duality for \texorpdfstring{$\ell_2$}{L2}-Regularized
Multiclass Classification Problems}
\label{sec:duality}

Let $X \in \Rb^{d \times n}$ be the matrix of training examples $x_i \in \Rb^d$,
 let $W \in \Rb^{d \times m}$ be the matrix of primal variables
obtained by stacking the %
vectors $w_y \in \Rb^d$, and
$A \in \Rb^{m \times n}$ the matrix of dual variables.

Before we prove our main result of this section
(Theorem~\ref{thm:primal-dual}), we first impose a technical constraint
on a loss function to be compatible with the choice of the
ground truth coordinate.
The top-$k$ hinge loss from Section \ref{sec:topkloss} satisfies
this requirement as we show in Proposition~\ref{prop:topk-compatible}.
We also prove an auxiliary Lemma~\ref{lem:conjugate},
which is then used in Theorem~\ref{thm:primal-dual}.

\begin{definition}\label{def:compatible}
A convex function $\phi$ is \emph{\compatible{j}} if for any
$y \in \Rb^m$ with $y_j = 0$ we have that
\begin{align*}
\sup \{ \inner{y, x} - \phi(x) \given x_j = 0 \} = \phi^*(y) .
\end{align*}
\end{definition}
This constraint is needed to prove equality in the following Lemma.
\begin{lemma}\label{lem:conjugate}
Let $\phi$ be \compatible{j},
let $H_j = \Id - \ones \tra{e}_j$, and let $\Phi(x) = \phi(H_j x)$, then
\begin{align*}
\Phi^*(y) = \begin{cases}
\phi^*(y - y_j e_j) & \text{if } \inner{\ones, y} = 0, \\
+\infty & \text{otherwise}.
\end{cases}
\end{align*}
\end{lemma}
\iflongversion
\begin{proof}
We have that
$\Ker H_j = \{ x \given H_j x = 0 \} = \{ t\ones \given t \in \Rb \}$
and
$\ort{\Ker} H_j = \{ x \given \inner{\ones, x} = 0 \}$.
\begin{align*}
\Phi^*(y) &= \sup \{ \inner{y,x} - \Phi(x) \given x \in \Rb^m \} \\
&= \sup\{ \innern{y,\para{x}} + \innern{y,\ort{x}} - \phi(H_j \ort{x})
\given x = \para{x} + \ort{x},
\para{x} \in \Ker H_j, \ort{x} \in \ort{\Ker} H_j \}.
\end{align*}
It follows that $\Phi^*(y)$ can only be finite if
$\innern{y,\para{x}} = 0$, which implies
$y \in \ort{\Ker} H_j$.
Let $\pinv{H}_j$ be the Moore-Penrose pseudoinverse of $H_j$.
For a $y \in \ort{\Ker} H_j$, we can write
\begin{equation}\label{eq:conjugate-ineq}
\begin{aligned}
\Phi^*(y) &=
\sup\{ \innern{y, \pinv{H}_j H_j \ort{x}} - \phi(H_j \ort{x})
\given \ort{x} \in \ort{\Ker} H_j \} \\
&=
\sup\{ \innern{\tra{(\pinv{H}_j)} y, z} - \phi(z)
\given z \in \Img H_j \} \\
&\leq
\sup\{ \innern{\tra{(\pinv{H}_j)} y, z} - \phi(z)
\given z \in \Rb^m \} = \phi^*(\tra{(\pinv{H}_j)} y) ,
\end{aligned}
\end{equation}
where $\Img H_j = \{ H_j x \given x \in \Rb^m \}$.
Using rank-$1$ update of the Moore-Penrose pseudoinverse
(\cite{petersen2008matrix}, \S~3.2.7), we can compute
$\tra{(\pinv{H}_j)} = \Id - e_j \tra{e}_j
- \frac{1}{m}(\ones - e_j) \tra{\ones}$.
Since $y \in \ort{\Ker} H_j$, the last term is zero and we have
$\tra{(\pinv{H}_j)} y = y - y_j e_j$.
Finally, we use the fact that $\phi$ is \compatible{j} to prove
that the inequality in (\ref{eq:conjugate-ineq}) is satisfied
with equality.
We have that
$\Img H_j = \{ z \given z_j = 0 \}$
and $(y - y_j e_j)_j = 0$.
Therefore, when $\inner{\ones, y} = 0$,
$\Phi^*(y)
= \sup\{ \innern{y - y_j e_j, z} - \phi(z) \given z_j = 0 \}
= \phi^*(y - y_j e_j)$.
\end{proof}
\fi

We can now use Lemma~\ref{lem:conjugate} to compute convex conjugates
of the loss functions.

\begin{theorem}\label{thm:primal-dual}
Let $\phi_i$ be \compatible{y_i} for each $i \in [n]$,
let $\lambda > 0$ be a regularization parameter, and
let $K = \tra{X}\!X$ be the Gram matrix.
The primal and Fenchel dual objective functions are given as:
\begin{equation*}%
\begin{aligned}
P(W) &= +\frac{1}{n} \sum_{i=1}^n
\phi_i \left( \tra{W}x_i - \inner{w_{y_i},x_i} \ones \right)
+ \frac{\lambda}{2} \tr\left(\tra{W} W \right) , \\
D(A) &= -\frac{1}{n} \sum_{i=1}^n
\phi_i^* \left( - \lambda n ( a_i - a_{y_i,i} e_{y_i}) \right)
- \frac{\lambda}{2} \tr\left( A K \tra{A} \right),
\text{ if } \inner{\ones, a_i} = 0 \; \forall i, \; +\infty
\text{ otherwise.}
\end{aligned}
\end{equation*}
Moreover, we have that
$W = X\tra{A}$ and $\tra{W}x_i = A K_i$,
where $K_i$ is the $i$-th column of $K$.
\end{theorem}
\iflongversion
\begin{proof}
We use Fenchel duality (see \eg \cite{borwein2000convex}, Theorem~3.3.5),
to write
$P(W) = g(\tra{X} W) + f(W)$, and
$D(A) = -g^*(-\tra{A}) - f^*(X\tra{A})$, for the functions
$g$ and $f$ defined as follows:
\begin{align*}
g(\tra{X} W) &= \frac{1}{n} \sum_{i=1}^n \Phi_i
\left(\tra{W} x_i \right) = \frac{1}{n} \sum_{i=1}^n \phi_i
\left( H_{y_i} \tra{W} x_i \right), &
f(W) &= \frac{\lambda}{2} \tr\left(\tra{W} W \right)
= \frac{\lambda}{2} \norms{W}_F,
\end{align*}
where $H_{y_i} = \Id - \ones \tra{e}_{y_i}$.
One can easily verify that
$g^*(-\tra{A}) = \frac{1}{n} \sum_{i=1}^n \Phi_i^*(-n a_i)$
and 
$f^*(X\tra{A}) = \frac{\lambda}{2} \norms{\frac{1}{\lambda} X\tra{A}}_F$.
From Lemma~\ref{lem:conjugate}, we have that
$
\Phi_i^*(-n a_i) = \phi^*(-n (a_i - a_{y_i,i} e_{y_i})) ,
$
if
$\inner{\ones, -n a_i} = 0$,
and $+\infty$ otherwise.
To complete the proof, we redefine
$A \leftarrow \frac{1}{\lambda} A$ for convenience,
and use the first order optimality condition
(\cite{borwein2000convex}, Ex.~9.f in \S~3)
for the $W = X\tra{A}$ formula.
\end{proof}
\fi

Finally, we show that Theorem~\ref{thm:primal-dual} applies to
the loss functions that we consider.

\begin{proposition}\label{prop:topk-compatible}
The top-$k$ hinge loss function from Section \ref{sec:topkloss} is \compatible{y_i}.
\end{proposition}
\iflongversion
\begin{proof}
Let $c = \ones - e_{y_i}$ and consider the %
loss $\phi_k$.
As in Proposition~\ref{prop:topk-pairs}, we have
\begin{align*}
\max_{a, \, a_{y_i} = 0} \{ \inner{a, b} - \phi_k(a) \} =
\min_{\lambda \in \Smk} \{ - \inner{c, \lambda}
+ \max_{a, \, a_{y_i} = 0} \inner{a, b - \lambda} \} =
\phi_k^*(b) ,
\end{align*}
where we used that $c_{y_i} = 0$ and $b_{y_i} = 0$
(\cf Definition~\ref{def:compatible}), \ie
the $y_i$-th coordinate has no influence.
\end{proof}
\fi
We have repeated the derivation from Section 5.7 in \cite{shalev2014accelerated}
as there is a typo in the optimization problem (20) leading to the 
conclusion that $a_{y_i,i}$ must be $0$ at the optimum. Lemma~\ref{lem:conjugate}
fixes this by making the requirement $a_{y_i,i} = - \sum_{j \neq y_i} a_{j,i}$ explicit.
Note that this modification is already mentioned in their pseudo-code
for Prox-SDCA.

\subsection{Optimization of \Topk\ Multiclass SVM via Prox-SDCA}
\label{sec:topk-sdca}

\input{figures/algorithm-sdca}
As an optimization scheme,
we employ the proximal stochastic dual coordinate ascent (Prox-SDCA) framework
of Shalev-Shwartz and Zhang \cite{shalev2014accelerated},
which has strong convergence guarantees and is easy to adapt to our problem.
In particular, we iteratively update a batch $a_i \in \Rb^m$
of dual variables corresponding to the training pair $(x_i,y_i)$,
so as to maximize the dual objective $D(A)$
from Theorem~\ref{thm:primal-dual}.
We also maintain the primal variables $W = X\tra{A}$
and stop when the relative duality gap
is below $\epsilon$.
This procedure is summarized in Algorithm~\ref{alg:topk-sdca}.

Let us make a few comments on the advantages of the proposed method.
First, apart from the update step which we discuss below,
all main operations can be computed using a BLAS library,
which makes the overall implementation efficient.
Second, the update step in Line~\ref{alg:topk-sdca:update-a}
is optimal in the sense that it yields
maximal dual objective increase
jointly over $m$ variables.
This is opposed to SGD updates with data-independent step sizes,
as well as to maximal but \emph{scalar} updates in other SDCA variants.
Finally, we have a well-defined stopping criterion
as we can compute the duality gap
(see discussion in \cite{bousquet2008tradeoffs}).
The latter is especially attractive if there is a time budget
for learning.
The algorithm can also be easily kernelized
since $\tra{W} x_i = A K_i$
(\cf Theorem~\ref{thm:primal-dual}).

\subsubsection{Dual Variables Update}
\label{sec:sdca:update}

For the proposed top-$k$  hinge loss from Section \ref{sec:topkloss},
optimization of the dual objective $D(A)$ over $a_i \in \Rb^m$
given other variables fixed is an instance of a regularized (biased)
projection problem onto the top-$k$ simplex
$\Smk(\frac{1}{\lambda n})$.
Let $\wo{a}{j}$ be obtained by removing the $j$-th coordinate
from vector $a$.

\begin{proposition}\label{pro:biasproj}
The following two problems are equivalent
with $\wo{a_i}{y_i} = -x$ and $a_{y_i,i} = \inner{\ones, x}$
\begin{align*}
&\max_{a_i} \{ D(A) \given \inner{\ones, a_i} = 0 \} \; \equiv \;
\min_{x} \{ \norms{b - x} + \rho \inner{\ones, x}^2
\given x \in \Smk(\tfrac{1}{\lambda n}) \} ,
\end{align*}
where
$b = \frac{1}{\inner{x_i,x_i}}%
\left( \wo{q}{y_i} + (1 - q_{y_i})\ones \right)$,
$q = \tra{W} x_i - \inner{x_i,x_i} a_i$ and
$\rho = 1$.%
\end{proposition}
\iflongversion
\begin{proof}
Using Proposition~\ref{prop:topk-pairs} 
and Theorem~\ref{thm:primal-dual}, we write
\begin{align*}
\max_{a_i} \{
&-\frac{1}{n} \phi_i^* \left(
- \lambda n ( a_i - a_{y_i,i} e_{y_i}) \right) 
-\frac{\lambda}{2} \tr\left( A K \tra{A} \right)
\given \inner{\ones, a_i} = 0 \} .
\end{align*}
For the loss function, we get
\begin{align*}
-\frac{1}{n} \phi_i^* \left(
- \lambda n ( a_i - a_{y_i,i} e_{y_i}) \right) = 
 \lambda a_{y_i,i} ,
\end{align*}
with $- \lambda n ( a_i - a_{y_i,i} e_{y_i}) \in \Smk$.
One can verify that the latter constraint is equivalent to
$-\wo{a_i}{y_i} \in \Smk(\frac{1}{\lambda n})$,
$a_{y_i,i} = \innern{\ones, -\wo{a_i}{y_i}}$.
Similarly, we write for the regularization term
\begin{align*}
\tr\left( A K \tra{A} \right) = 
K_{ii} \inner{a_i,a_i} + 2 \sum_{j \neq i} K_{ij} \inner{a_i,a_j}
+ {\rm const} ,
\end{align*}
where the ${\rm const}$ does not depend on $a_i$.
Note that
$\sum_{j \neq i} K_{ij} a_j = A K_i - K_{ii} a_i = q$
and can be computed using the ``old'' $a_i$.
Let $x \bydef -\wo{a_i}{y_i}$, we have
\begin{align*}
\inner{a_i,a_i} &= \inner{\ones,x}^2 + \inner{x,x}, &
\inner{q,a_i} &= q_{y_i} \inner{\ones,x} - \innern{\wo{q}{y_i},x}.
\end{align*}
Plugging everything together and multiplying with
$-2/\lambda$, we obtain
\begin{align*}
\min_{x \in \Smk(\frac{1}{\lambda n})}
- 2 \inner{\ones, x}
+ 2 \big(
q_{y_i} \inner{\ones,x} - \innern{\wo{q}{y_i},x} \big)
+ K_{ii} \big(\inner{\ones,x}^2 + \inner{x,x} \big) .
\end{align*}
Collecting the corresponding terms finishes the proof.
\end{proof}
\fi

We discuss in the following section
how to project onto the set $\Smk(\frac{1}{\lambda n})$ efficiently.

%% file: figures/algorithm-sdca.tex
\begin{wrapfigure}[16]{R}{0.5\textwidth}\vspace*{-1.3em}
\begin{minipage}[c]{.95\linewidth}%
\setlength{\intextsep}{.25em}%
\centering%
\begin{algorithm}[H]
\caption{Top-$k$ Multiclass SVM}
\label{alg:topk-sdca}
\begin{algorithmic}[1]
\footnotesize
  \STATE {\bfseries Input:}
training data $\{(x_i,y_i)_{i=1}^n\}$,
parameters $k$ (loss), %
$\lambda$ (regularization),
$\epsilon$ (stopping cond.)
  \STATE {\bfseries Output:}
$W \in \Rb^{d \times m}$, $A \in \Rb^{m \times n}$
  \STATE {\bfseries Initialize:}
$W \leftarrow 0$, $A \leftarrow 0$
  \REPEAT
    \STATE randomly permute training data
    \FOR{$i=1$ \TO $n$}
      \STATE $s_i \leftarrow \tra{W} x_i$
\COMMENT{prediction scores}
\label{alg:topk-sdca:scores}
      \STATE $a_i^{\rm old} \leftarrow a_i$
\COMMENT{cache previous values}
      \STATE
$a_i \leftarrow update(k,\lambda, \norms{x_i}, y_i, s_i, a_i)$ \\
\quad \COMMENT{see \S~\ref{sec:sdca:update} for details}
\label{alg:topk-sdca:update-a}
      \STATE $W \leftarrow W + x_i \tra{(a_i - a_i^{\rm old})}$ \\
\quad \COMMENT{rank-$1$ update}
\label{alg:topk-sdca:update-w}
    \ENDFOR
  \UNTIL{relative duality gap is below $\epsilon$}
\end{algorithmic}
\end{algorithm}
\end{minipage}
\end{wrapfigure}

%% file: sections/projection.tex
One of our main technical results is an algorithm
for efficiently computing projections onto $\Smk(r)$,
respectively the biased projection introduced in
Proposition~\ref{pro:biasproj}.
The optimization problem in Proposition~\ref{pro:biasproj}
reduces to the Euclidean projection onto $\Smk(r)$ for $\rho=0$,
and for $\rho>0$ it biases the solution
to be orthogonal to $\ones$.
Let us highlight that $\Smk(r)$ is substantially different from the standard simplex and none of the existing methods
can be used as we discuss below.

\subsection{Continuous Quadratic Knapsack Problem}
\label{sec:proj-knapsack}

Finding the Euclidean projection onto the simplex
is an instance of the general optimization problem
$\min_x \{ \norms{a-x}_2 \given \inner{b, x} \leq r, \;
l \leq x_i \leq u \}$
known as the \emph{continuous quadratic knapsack problem} (CQKP).
For example, to project onto the simplex we set
$b=\ones$, $l=0$ and $r=u=1$.
This is a well examined problem
and several highly efficient algorithms are available (see the surveys \cite{patriksson2008survey,patriksson2015algorithms}).
The first main difference to our set is the upper bound on the $x_i$'s.
All existing algorithms expect that $u$ is \emph{fixed},
which allows them to consider decompositions
$\min_{x_i} \{ (a_i -x_i)^2 \given l \leq x_i \leq u \}$
which can be solved in closed-form.
In our case, the upper bound $\tfrac{1}{k} \inner{\ones,x}$
introduces coupling across all variables,
which makes the existing algorithms not applicable.
A second main difference is the bias term $\rho\inner{\ones,x}^2$
added to the objective. The additional difficulty 
introduced by this term is relatively minor.
Thus we solve the problem for general $\rho$
(including $\rho=0$ for the Euclidean projection onto $\Smk(r)$) even
though we need only $\rho=1$ in Proposition~\ref{pro:biasproj}.
The only case when our problem reduces to CQKP is when
the constraint $\inner{\ones, x} \leq r$ is satisfied with equality.
In that case we can let $u=r/k$ and use any algorithm
for the knapsack problem.
We choose \cite{kiwiel2008variable} since it is easy to implement,
does not require sorting, and scales linearly in practice.
The bias in the projection problem reduces
to a constant $\rho r^2$ in this case and has, therefore, no effect.

\subsection{Projection onto the \Topk\ Cone}
\label{sec:proj-cone}

When the constraint $\inner{\ones, x} \leq r$
is not satisfied with equality at the optimum,
it has essentially no influence on the projection problem
and can be removed.
In that case we are left with the problem of the (biased) projection onto
the top-$k$ cone which we address with the following lemma.

\begin{lemma}\label{lem:proj-topk-cone}
Let $x^* \in \Rb^d$ be the solution to the following optimization problem
\begin{align*}
\min_x \{ \norms{a-x} + \rho \inner{\ones, x}^2 \given
0 \leq x_i \leq \tfrac{1}{k} \inner{\ones, x}, \; i \in [d] \} ,
\end{align*}
and let
$U \bydef \{ i \given x^*_i = \tfrac{1}{k} \inner{\ones, x^*} \}$,\,
$M \bydef \{ i \given 0 < x^*_i < \tfrac{1}{k} \inner{\ones, x^*} \}$,\,
$L \bydef \{ i \given x^*_i = 0 \}$.
\begin{enumerate}
\item If $U = \varnothing$ and $M = \varnothing$, then $x^* = 0$.
\item If $U \neq \varnothing$ and $M = \varnothing$, then 
$U = \{[1], \ldots, [k]\}$,
$x^*_i = \tfrac{1}{k+\rho k^2} \sum_{i=1}^k a_{[i]}$ for $i \in U$,
where $[i]$ is the index of the $i$-th largest component in $a$. 
\item Otherwise ($M \neq \varnothing$),
the following system of linear equations holds
\begin{align}\label{lem:proj-topk-cone:sys}
\begin{cases}
u &= \big(
\abs{M} \sum_{i \in U} a_i + (k - \abs{U}) \sum_{i \in M} a_i
\big) / D , \\
t' &= \big(
\abs{U} (1 + \rho k) \sum_{i \in M} a_i
- (k - \abs{U} + \rho k \abs{M}) \sum_{i \in U} a_i
\big) / D , \\
D &= (k - \abs{U})^2 + (\abs{U} + \rho k^2) \abs{M} ,
\end{cases}
\end{align}
together with the feasibility constraints
on $t \bydef t' + \rho u k$
\begin{align}\label{lem:proj-topk-cone:constr}
\max_{i \in L} a_i &\leq t \leq \min_{i \in M} a_i , &
\max_{i \in M} a_i &\leq t + u \leq \min_{i \in U} a_i ,
\end{align}
and we have
$x^* = \min\{ \max\{0, a - t \}, u \}$.
\end{enumerate}
\end{lemma}
\iflongversion
\begin{proof}
We consider an equivalent problem
\begin{align*}
\min_{x,s} \{ \tfrac{1}{2}\norms{a-x} + \tfrac{1}{2}\rho s^2 \given
\inner{\ones, x} = s, \;
0 \leq x_i \leq \tfrac{s}{k}, \; i \in [d] \} .
\end{align*}
Let $t$, $\mu_i \geq 0$, $\nu_i \geq 0$ be the dual variables,
and let $\Lc$ be the Lagrangian:
\begin{align*}
\Lc(x,s,t,\mu,\nu) = \tfrac{1}{2}\norms{a-x} + \tfrac{1}{2}\rho s^2
+ t(\inner{\ones, x} - s) - \inner{\mu, x}
+ \inner{\nu, x - \tfrac{s}{k} \ones} .
\end{align*}
From the KKT conditions, we have that
\begin{align*}
\partial_x \Lc &= x - a + t\ones - \mu + \nu = 0, &
\partial_s \Lc &= \rho s - t - \tfrac{1}{k} \inner{\ones, \nu} = 0, &
\mu_i x_i &= 0, &
\nu_i (x_i - \tfrac{s}{k}) &= 0 .
\end{align*}
We have that
$x_i = \min\{ \max\{0, a_i - t \}, \frac{s}{k} \}$,
$\nu_i = \max\{0, a_i -t - \frac{s}{k} \}$,
and $s = \frac{1}{\rho}(t + \frac{1}{k} \inner{\ones, \nu})$.
Let $p \bydef \inner{\ones, \nu}$.
We have $t = \rho s - \frac{p}{k}$.
Using the definition of the sets $U$ and $M$, we get
\begin{align*}
s &= \sum_{i \in U} \frac{s}{k} + \sum_{i \in M} (a_i - t) =
\sum_{i \in M} a_i - \abs{M}(\rho s - \frac{p}{k}) + \abs{U} \frac{s}{k}, \\
p &= \sum_{i \in U} (a_i - t - \frac{s}{k}) =
\sum_{i \in U} a_i - \abs{U}(\rho s - \frac{p}{k}) - \abs{U} \frac{s}{k}.
\end{align*}
In the case $U\neq \varnothing$ and $M=\varnothing$ we get the simplified equations
\begin{align*} 
s&=\sum_{i \in U} \frac{s}{k} = |U|\frac{s}{k}  \quad \Longrightarrow \quad |U|=k,\\
p&= \sum_{i \in U} a_i - k\rho s + p - s \quad \Longrightarrow \quad x_i = \frac{s}{k}=\frac{1}{k+\rho \,k^2}\sum_{i \in U} a_i, \; i \in U.
\end{align*}
In the remaining case 
solving this system for $u \bydef \frac{s}{k}$
and $t' \bydef - \frac{p}{k}$,
we get exactly the system in (\ref{lem:proj-topk-cone:sys}).
The constraints (\ref{lem:proj-topk-cone:constr})
follow from the definition of the sets $U$, $M$, $L$, and
ensure that the computed thresholds $(t,u)$ are compatible
with the corresponding partitioning of the index set.
\end{proof}
\fi

We now show how to check if the (biased) projection is $0$.
For the standard simplex, where the cone
is the positive orthant $\Rb^d_+$, the projection is $0$
when all $a_i \leq 0$.
It is slightly more involved for $\Smk$.

\begin{lemma}\label{lem:proj-topk-cone-zero}
The biased projection $x^*$ onto the top-$k$ cone is zero if 
$\sum_{i=1}^k a_{[i]} \leq 0$ (sufficient condition).
If $\rho=0$ this is also necessary.
\end{lemma}
\iflongversion
\begin{proof}
Let $K \bydef \{ x \given 0 \leq x_i \leq \frac{1}{k} \inner{\ones, x} \}$
be the top-$k$ cone.
It is known that the Euclidean projection of $a$ onto $K$
is $0$ if and only if $a \in N_K(0) \bydef
\{y \given \forall x \in K, \, \inner{y,x} \leq 0 \}$,
\ie $a$ is in the normal cone to $K$ at $0$.
Therefore, we obtain as an equivalent condition that
$\max_{x \in K} \inner{a,x} \leq 0$.
Take any $x \in K$ and let $s = \inner{\ones, x}$.
If $s > 0$, we have that at least $k$ components in $x$ must be positive.
To maximize $\inner{a,x}$, we would have exactly $k$ positive
$x_i = \frac{s}{k}$ corresponding to the $k$ largest components in $a$.
That would result in $\inner{a,x} = \frac{s}{k} \sum_{i=1}^k a_{[i]}$,
which is non-positive if and only if $\sum_{i=1}^k a_{[i]} \leq 0$.

For $\rho>0$, the
objective function has an additional term
$\rho \inner{\ones, x}^2$ that vanishes at $x=0$.
Therefore, if $x=0$ is optimal for the Euclidean projection,
it must also be optimal for the biased projection.
\end{proof}
\fi

\textbf{Projection.}
Lemmas~\ref{lem:proj-topk-cone} and \ref{lem:proj-topk-cone-zero}
suggest a simple algorithm for the (biased) projection onto the top-$k$ cone.
First, we check if the projection is constant (cases $1$ and $2$
in Lemma~\ref{lem:proj-topk-cone}).
In case $2$, %
we compute $x$ and check if it
is compatible with the corresponding sets $U$, $M$, $L$.
In the general case~$3$, we suggest a simple exhaustive search
strategy. We sort $a$ and loop over the feasible partitions $U$, $M$, $L$
until we find a solution to (\ref{lem:proj-topk-cone:sys})
that satisfies (\ref{lem:proj-topk-cone:constr}).
Since we know that $0 \leq \abs{U} < k$ and
$k \leq \abs{U} + \abs{M} \leq d$,
we can limit the search to $(k-1)(d-k+1)$ iterations in the worst case,
where each iteration requires a constant number of operations.
For the biased projection, we leave $x=0$ as the fallback case
as Lemma~\ref{lem:proj-topk-cone-zero} gives only a sufficient
condition.
This yields a runtime complexity of $O(d \log(d) + kd)$,
which is comparable to simplex projection algorithms based on sorting.

\subsection{Projection onto the \Topk\ Simplex}
\label{sec:proj-simplex}

As we argued in \S~\ref{sec:proj-knapsack},
the (biased) projection onto the top-$k$ simplex
becomes either the knapsack problem or the
(biased) projection onto the top-$k$ cone
depending on the constraint 
$\inner{\ones, x} \leq r$ at the optimum.
The following Lemma provides a way to check
which of the two cases apply.

\begin{lemma}\label{lem:proj-topk-simplex}
Let $x^* \in \Rb^d$ be the solution to the following optimization problem
\begin{align*}
\min_x \{ \norms{a-x} + \rho \inner{\ones, x}^2 \given
\inner{\ones, x} \leq r, \;
0 \leq x_i \leq \tfrac{1}{k} \inner{\ones, x}, \; i \in [d] \} ,
\end{align*}
let $(t,u)$ be the optimal thresholds such that
$x^* = \min\{ \max\{0, a - t \}, u \}$,
and let $U$ be defined as in Lemma~\ref{lem:proj-topk-cone}.
Then it must hold that
$\lambda = t + \frac{p}{k} - \rho r \geq 0$,
where
$p = \sum_{i \in U} a_i - \abs{U}(t + u)$.
\end{lemma}
\iflongversion
\begin{proof}
As in Lemma~\ref{lem:proj-topk-cone}, we consider an equivalent problem
\begin{align*}
\min_{x,s} \{ \tfrac{1}{2}\norms{a-x} + \tfrac{1}{2}\rho s^2 \given
\inner{\ones, x} = s, \; s \leq r, \;
0 \leq x_i \leq \tfrac{s}{k}, \; i \in [d] \} .
\end{align*}
Let $t$, $\lambda \geq 0$, $\mu_i \geq 0$, $\nu_i \geq 0$
be the dual variables,
and let $\Lc$ be the Lagrangian:
\begin{align*}
\Lc = \tfrac{1}{2}\norms{a-x} + \tfrac{1}{2}\rho s^2
+ t(\inner{\ones, x} - s) + \lambda (s - r) - \inner{\mu, x}
+ \inner{\nu, x - \tfrac{s}{k} \ones} .
\end{align*}
From the KKT conditions, we have that
\begin{align*}
\partial_x \Lc = x - a + t\ones - \mu + \nu &= 0, &
\partial_s \Lc = \rho s - t + \lambda - \tfrac{1}{k} \inner{\ones, \nu} &= 0,
&& \\
\mu_i x_i &= 0, &
\nu_i (x_i - \tfrac{s}{k}) &= 0 , &
\lambda (s - r) &= 0.
\end{align*}
If $s < r$, then $\lambda = 0$ and we recover the top-$k$ cone
problem of Lemma~\ref{lem:proj-topk-cone}.
Otherwise, we have that $s = r$ and
$\lambda = t + \tfrac{1}{k} \inner{\ones, \nu} - \rho r \geq 0$.
The fact that $\nu_i = \max\{0, a_i -t - u \}$,
where $u = \frac{r}{k}$, completes the proof.
\end{proof}
\fi

\textbf{Projection.}
We can now use Lemma~\ref{lem:proj-topk-simplex}
to compute the (biased) projection onto $\Smk(r)$ as follows.
First, we check the special cases of zero and constant projections,
as we did before.
If that fails, we proceed with the knapsack problem
since it is faster to solve.
Having the thresholds $(t,u)$ and the partitioning into the sets
$U$, $M$, $L$, we compute the value of $\lambda$ as given
in Lemma~\ref{lem:proj-topk-simplex}.
If $\lambda \geq 0$, we are done.
Otherwise, we know that $\inner{\ones, x} < r$ and go directly
to the general case $3$ in Lemma~\ref{lem:proj-topk-cone}.

%% file: sections/usunier.tex
In this section we show how the Usunier version of the top-$k$ hinge loss
(\ref{eq:topk-usu}) can be optimized using the Prox-SDCA framework
from \S~\ref{sec:optimization}.
The two main ingredients that we discuss are the conjugate loss
and the (biased) projection.
It turns out that the only difference between
the conjugate of the top-$k$ hinge loss (\ref{eq:topk-loss}) introduced above
and the conjugate of (\ref{eq:topk-usu}) are their
effective domains.

\begin{proposition}\label{prop:topk-usu-pairs}
A primal-conjugate pair for the top-$k$ Usunier loss (\ref{eq:topk-usu}) is
\begin{align}\label{eq:topk-usu-pair}
\tilde{\phi}_k(a) &=
\frac{1}{k} \sum_{j=1}^k \max \Big\{0, (a + c)_{[j]} \Big\} , &
\tilde{\phi}_k^*(b) &=
\begin{cases}
- \inner{c, b} & \text{if } b \in \SUmk , \\
+\infty & \text{otherwise} ,
\end{cases}
\end{align}
where
\begin{align*}
\SUmk(r) &\bydef
\left\{ x \given \inner{\ones, x} \leq r, \;
0 \leq x_i \leq \tfrac{1}{k} ,
\; i \in [m] \right\} .
\end{align*}
Moreover,
$
\tilde{\phi}_k(a) = \max \{ \inner{a + c, \lambda} \given \lambda \in \SUmk \}
$.
\end{proposition}
\begin{proof}
The proof is similar to the proof of Proposition~\ref{prop:topk-pairs};
the main step is as follows:
\begin{align*}
\tilde{\phi}_k(a) &= \min_{t,\xi,h} \big\{ t + \tfrac{1}{k} \inner{\ones, \xi} \, \given \,
\xi_j \geq h_j - t, \;
\xi_j \geq 0, \;
h_j \geq a_j + c_j, \;
h_j \geq 0 \big\} \\
&= \max_{\lambda} \big\{
\inner{a + c, \lambda} \, \given \,
\inner{\ones, \lambda} \leq 1, \;
0 \leq \lambda_j \leq \tfrac{1}{k}
\big\} .
\end{align*}
\end{proof}

Note that the upper bounds on $x_i$'s are now fixed to $1/k$,
which means the Euclidean projection onto the set $\SUmk$ is an instance
of the continuous quadratic knapsack problem from \S~\ref{sec:proj-knapsack}.
Unfortunately, the proximal step in the SDCA framework corresponds
to a \emph{biased} projection where there is an additional $\ell_2$ regularizer
on the sum $\inner{\ones, x}$ coming from the regularizer in the training objective.
To address this issue, we follow the derivation given in the proofs of
Lemmas~\ref{lem:proj-topk-cone} and \ref{lem:proj-topk-simplex}.

The update step for the top-$k$ Usunier loss (\ref{eq:topk-usu}) is equivalent to
(with $l=0$ and $u=1/k$):
\begin{align*}
\min_{x,s} \{ \tfrac{1}{2}\norms{a-x} + \tfrac{1}{2}\rho s^2 \given
\inner{\ones, x} = s, \; s \leq r, \;
l \leq x_i \leq u, \; i \in [d] \} .
\end{align*}
Let $t$, $\lambda \geq 0$, $\mu_i \geq 0$, $\nu_i \geq 0$
be the dual variables,
and let $\Lc$ be the Lagrangian:
\begin{align*}
\Lc = \tfrac{1}{2}\norms{a-x} + \tfrac{1}{2}\rho s^2
+ t(\inner{\ones, x} - s) + \lambda (s - r) - \inner{\mu, l \ones - x}
+ \inner{\nu, x - u \ones} .
\end{align*}
From the KKT conditions, we have that
\begin{align*}
\partial_x \Lc = x - a + t\ones - \mu + \nu &= 0, &
\partial_s \Lc = \rho s - t + \lambda &= 0,
&& \\
\mu_i (l - x_i) &= 0, &
\nu_i (x_i - u) &= 0, &
\lambda (s - r) &= 0,
\end{align*}
which then leads to
\begin{align*}
x &= a - t\ones + \mu - \nu = \min\{ \max \{l, x - t\}, u \} , &
\lambda &= t - \rho s.
\end{align*}
Now, we can do case distinction based on the sign of $\lambda$.
If $\lambda > 0$, then $\inner{\ones, x} = s = r$ and $t > \rho r$.
In this case $\tfrac{1}{2}\rho s^2 = \tfrac{1}{2}\rho r^2 \equiv \rm const$,
therefore this term can be ignored and we get the knapsack problem
from \S~\ref{sec:proj-knapsack}.
Otherwise, if $s < r$, then $\lambda = 0$ and $t=\rho s$.
Using the index sets $U$, $M$ and $L$ as in Lemma~\ref{lem:proj-topk-cone},
we have that
\begin{align*}
t = \rho \Big( \sum_L l + \sum_M (a_i - t) + \sum_U u \Big)
= \rho \Big( l \abs{L} + u \abs{U} - t \abs{M} + \sum_M a_i \Big).
\end{align*}
Solving for $t$ with $\rho > 0$, we obtain that
\begin{align}\label{eq:proj-usu-t}
t = \Big( l \abs{L} + u \abs{U} + \sum_M a_i \Big)
/ \Big( \frac{1}{\rho} + \abs{M} \Big).
\end{align}

\textbf{Projection.}
To compute the (biased) projection,
we follow the same steps as in \S~\ref{sec:proj-simplex}.
First, we solve the knapsack problem using the algorithm of \cite{kiwiel2008variable},
which also computes the dual variable $t$.
If $t > \rho r$, then we are done;
otherwise, we sort $a$ and loop over the feasible index sets $U$, $M$, and $L$.
We stop once we find a $t$ that satisfies (\ref{eq:proj-usu-t}) and
is compatible with the corresponding index sets.

%% file: sections/experiments.tex
\input{figures/table-all}
We have two main goals in the experiments.
First, we show
that the (biased) projection onto the top-$k$ simplex
is scalable and
comparable to an efficient algorithm \cite{kiwiel2008variable}
\iflongversion
for the simplex projection.
\fi
Second, we show that the top-$k$ multiclass SVM
using both versions of the top-$k$ hinge loss
(\ref{eq:topk-loss}) and (\ref{eq:topk-usu}),
denoted \MethodSvmTopK{k} and \MethodUsuTopK{k} respectively,
leads to improvements in top-$k$ accuracy
consistently over all datasets and choices of $k$.
In particular, we note improvements compared to the multiclass SVM
of Crammer and Singer \cite{crammer2001algorithmic},
which corresponds to \MethodSvmTopK{1}/\MethodUsuTopK{1}.
We release our implementation of the projection procedures
and both SDCA solvers as a C++
library\footnote{\url{https://github.com/mlapin/libsdca}}
with a Matlab interface.

\iflongversion
\subsection{Scaling of the Projection onto the \Topk\ Simplex}

\input{figures/plot-scaling}
We follow the experimental setup of \cite{liu2009efficient}.
We sample $1000$ points from the normal distribution $\Nc(0,1)$
and solve the projection %
problems using
the algorithm of
\cite{kiwiel2008variable} (denoted as Knapsack)
and using our proposed method of projecting onto
the set $\Smk$ for different values of $k=1,5,10$.
We report the total CPU time taken on a single
Intel(R) Xeon(R) 2.20GHz processor.
As one can see, the scaling is linear in the problem dimension
and the run times are essentially the same.

\fi

\subsection{Image Classification Experiments}
\label{sec:classification}

We evaluate our method on five image classification datasets
of different scale and complexity:
Caltech 101 Silhouettes \cite{swersky2012probabilistic}
($m=101$, $n=4100$),
MIT Indoor 67 \cite{quattoni2009recognizing}
($m=67$, $n=5354$),
SUN 397 \cite{xiao2010sun}
($m=397$, $n=19850$),
Places 205 \cite{zhou2014learning}
($m=205$, $n=2448873$),
and ImageNet 2012 \cite{ILSVRCarxiv14}
($m=1000$, $n=1281167$).
For Caltech, $d=784$, and for the others $d=4096$.
The results on the two large scale datasets are in the supplement.

We cross-validate hyper-parameters in the range
$10^{-5}$ to $10^3$, extending it when the optimal value is
at the boundary.
We use LibLinear \cite{REF08a} for \MethodSvmOva,
\MethodSvmPerf\ \cite{joachims2005support} with the corresponding
loss function for \MethodSvmRecKMulti{k},
and the code provided by \cite{li2014top} for \MethodTopPushMulti.
When a ranking method like \MethodSvmRecKMulti{k} and
\MethodTopPushMulti\ does not scale to a particular dataset
using the reduction of the multiclass to a binary problem
discussed in \S~\ref{sec:comparison},
we use the one-vs-all version of the corresponding method.
We implemented Wsabie$^{++}$ (denoted \MethodWsabie{m}{Q})
based on the pseudo-code from Table~3 in \cite{Gupta2014}.
\iflongversion
Among the baseline methods that we tried,
only \MethodTopPushOva\ scaled to the Places and the ImageNet datasets
both time and memory-wise\footnote{
LibLinear, although being generally fast, required too much memory
for the experiments to be feasible.
}.
\fi

On Caltech 101, we use features provided
by \cite{swersky2012probabilistic}.
For the other datasets, we extract CNN features of a pre-trained CNN
(fc7 layer after ReLU).
For the scene recognition datasets,
we use the Places 205 CNN \cite{zhou2014learning}
and for ILSVRC 2012 we use the Caffe reference model
\cite{jia2014caffe}.

Experimental results are given in
\iflongversion
Tables~\ref{tbl:small},~\ref{tbl:all}.
\fi
First, we note that our method is scalable to
large datasets with millions of training examples,
such as Places 205 and ILSVRC 2012
(results in the supplement).
Second, we observe that optimizing the top-$k$ hinge loss
(both versions)
yields consistently better top-$k$ performance.
This might come at the cost of a decreased top-$1$ accuracy
(\eg on MIT Indoor 67),
but, interestingly, may also result in a noticeable increase in the top-$1$ accuracy
on larger datasets like Caltech 101 Silhouettes and SUN 397.
This resonates with our argumentation that optimizing for top-$k$
is often more appropriate for datasets with a large number of classes.

Overall, we get systematic increase in top-$k$ accuracy
over all datasets that we examined.
For example,
we get the following improvements in top-$5$ accuracy
with our \MethodSvmTopK{10} compared to %
\MethodSvmTopK{1}:
$+2.6\%$ on Caltech 101,
$+1.2\%$ on MIT Indoor 67,
and $+2.5\%$ on SUN 397.

%% file: figures/table-all.tex
\begin{table}[ht]\scriptsize\centering\setlength{\tabcolsep}{.4em}
\begin{tabular}{l|cccccc||cccccc}
\multicolumn{1}{c}{} &
\multicolumn{6}{c}{\textbf{\small Caltech 101 Silhouettes}} &
\multicolumn{6}{c}{\textbf{\small MIT Indoor 67}} \\\toprule
Method &
Top-1 & Top-2 & Top-3 & Top-4 & Top-5 & Top-10 &
Method & Top-1 & Method & Top-1 & Method & Top-1 \\
\midrule
\midrule
\MethodTopK{1} \cite{swersky2012probabilistic} &  $62.1$ & - & $79.6$ & - & $83.1$ & - &
BLH \cite{bu2013superpixel} & $48.3$ &
DGE \cite{doersch2013mid} & $66.87$ &
RAS \cite{razavian2014cnn} & $69.0$ \\
\MethodTopK{2} \cite{swersky2012probabilistic} &  $61.4$ & - & $79.2$ & - & $83.4$ & - &
SP \cite{sun2013learning} & $51.4$ &
ZLX \cite{zhou2014learning} & $68.24$ &
KL \cite{koskela2014convolutional} & $70.1$ \\
\MethodTopK{5} \cite{swersky2012probabilistic} &  $60.2$ & - & $78.7$ & - & $83.4$ & - &
JVJ \cite{juneja2013blocks} & $63.10$ &
GWG \cite{gong2014multi} & $68.88$ & \\
\midrule
\midrule
Method &
Top-1 & Top-2 & Top-3 & Top-4 & Top-5 & Top-10 &
Top-1 & Top-2 & Top-3 & Top-4 & Top-5 & Top-10 \\
\midrule
\midrule
\MethodSvmOva & $61.81$ & $73.13$ & $76.25$ & $77.76$ & $78.89$ & $83.57$ &
$71.72$ & $81.49$ & $84.93$ & $86.49$ & $87.39$ & $90.45$ \\
\MethodTopPushMulti & $63.11$ & $75.16$ & $78.46$ & $80.19$ & $81.97$ & $86.95$ &
$70.52$ & $83.13$ & $86.94$ & $90.00$ & $91.64$ & $95.90$ \\
\iflongversion
\midrule
\MethodSvmPrecKMulti{1} & $61.29$ & $73.26$ & $76.12$ & $77.76$ & $79.11$ & $83.27$ &
$69.03$ & $80.67$ & $85.00$ & $87.16$ & $88.21$ & $91.87$ \\
\MethodSvmPrecKMulti{5} & $61.73$ & $73.99$ & $76.90$ & $78.50$ & $79.63$ & $84.22$ &
$69.18$ & $81.42$ & $85.45$ & $87.61$ & $88.43$ & $91.87$ \\
\MethodSvmPrecKMulti{10} & $61.90$ & $73.95$ & $76.68$ & $78.46$ & $79.67$ & $84.14$ &
$69.18$ & $81.42$ & $85.45$ & $87.61$ & $88.43$ & $91.87$ \\
\fi
\midrule
\MethodSvmRecKMulti{1} & $61.55$ & $73.13$ & $77.03$ & $79.41$ & $80.97$ & $85.18$ &
$71.57$ & $83.06$ & $87.69$ & $90.45$ & $92.24$ & $96.19$ \\
\iflongversion
\MethodSvmRecKMulti{2} & $61.25$ & $73.00$ & $76.33$ & $77.94$ & $79.15$ & $83.49$ &
$71.42$ & $81.49$ & $85.60$ & $87.24$ & $88.36$ & $92.16$ \\
\MethodSvmRecKMulti{3} & $61.51$ & $72.95$ & $76.55$ & $78.72$ & $80.49$ & $84.74$ &
$71.42$ & $81.57$ & $85.67$ & $87.39$ & $88.43$ & $92.24$ \\
\MethodSvmRecKMulti{4} & $61.55$ & $72.95$ & $76.68$ & $78.80$ & $80.58$ & $84.70$ &
$71.42$ & $81.57$ & $85.67$ & $87.24$ & $88.28$ & $92.01$ \\
\fi
\MethodSvmRecKMulti{5} & $61.60$ & $72.87$ & $76.51$ & $78.76$ & $80.54$ & $84.74$ &
$71.49$ & $81.49$ & $85.45$ & $87.24$ & $88.21$ & $92.01$ \\
\MethodSvmRecKMulti{10} & $61.51$ & $72.95$ & $76.46$ & $78.72$ & $80.54$ & $84.92$ &
$71.42$ & $81.49$ & $85.52$ & $87.24$ & $88.28$ & $92.16$ \\
\midrule
\iflongversion
\MethodWsabie{m}{0} & $62.33$ & $74.95$ & $78.59$ & $81.45$ & $83.66$ & $89.08$ &
$69.33$ & $83.06$ & $88.66$ & $91.72$ & $93.43$ & $97.54$ \\
\MethodWsabie{m}{1} & $59.69$ & $65.97$ & $68.92$ & $71.61$ & $73.82$ & $80.88$ &
$67.39$ & $80.15$ & $85.22$ & $88.88$ & $90.90$ & $95.90$ \\
\MethodWsabie{m}{2} & $57.39$ & $64.33$ & $67.88$ & $70.13$ & $71.95$ & $77.59$ &
$62.61$ & $76.57$ & $82.39$ & $86.19$ & $88.36$ & $93.81$ \\
\MethodWsabie{m}{4} & $56.78$ & $63.94$ & $67.36$ & $70.05$ & $72.08$ & $78.76$ &
$63.13$ & $76.87$ & $82.24$ & $85.67$ & $88.43$ & $94.63$ \\
\MethodWsabie{m}{8} & $57.17$ & $63.50$ & $67.01$ & $69.79$ & $71.87$ & $77.85$ &
$63.73$ & $77.24$ & $83.36$ & $86.87$ & $89.10$ & $94.63$ \\
\midrule
\MethodWsabie{192}{0} & $62.29$ & $76.25$ & $79.71$ & $81.40$ & $83.09$ & $88.17$ &
$69.78$ & $82.99$ & $88.36$ & $91.49$ & $93.51$ & $97.31$ \\
\MethodWsabie{192}{1} & $59.56$ & $65.97$ & $69.44$ & $71.65$ & $73.91$ & $79.45$ &
$67.24$ & $81.34$ & $85.60$ & $89.03$ & $91.19$ & $95.75$ \\
\MethodWsabie{192}{2} & $56.78$ & $63.29$ & $67.10$ & $69.87$ & $71.69$ & $78.37$ &
$63.28$ & $77.61$ & $84.03$ & $87.99$ & $89.93$ & $94.85$ \\
\MethodWsabie{192}{4} & $58.13$ & $64.37$ & $67.62$ & $69.92$ & $71.56$ & $78.15$ &
$62.54$ & $76.79$ & $84.10$ & $87.61$ & $89.18$ & $94.03$ \\
\MethodWsabie{192}{8} & $57.04$ & $66.28$ & $70.18$ & $73.39$ & $75.34$ & $82.79$ &
$63.06$ & $77.84$ & $84.55$ & $88.06$ & $90.37$ & $94.70$ \\
\midrule
\fi
\MethodWsabie{256}{0} & $62.68$ & $76.33$ & $79.41$ & $81.71$ & $83.18$ & $88.95$ &
$70.07$ & $84.10$ & $89.48$ & $92.46$ & $94.48$ & $\mathbf{97.91}$ \\
\MethodWsabie{256}{1} & $59.25$ & $65.63$ & $69.22$ & $71.09$ & $72.95$ & $79.71$ &
$68.13$ & $81.49$ & $86.64$ & $89.63$ & $91.42$ & $95.45$ \\
\MethodWsabie{256}{2} & $55.09$ & $61.81$ & $66.02$ & $68.88$ & $70.61$ & $76.59$ &
$64.63$ & $78.43$ & $84.18$ & $88.13$ & $89.93$ & $94.55$ \\
\iflongversion
\MethodWsabie{256}{4} & $56.52$ & $62.29$ & $65.76$ & $68.01$ & $70.13$ & $76.59$ &
$60.90$ & $75.97$ & $82.84$ & $86.79$ & $89.63$ & $94.63$ \\
\MethodWsabie{256}{8} & $55.79$ & $61.60$ & $65.58$ & $68.23$ & $70.39$ & $77.55$ &
$62.39$ & $75.15$ & $81.42$ & $85.82$ & $88.88$ & $94.03$ \\
\fi
\midrule
\midrule
\iflongversion
\MethodSvmTopK{1} & $62.81$ & $74.60$ & $77.76$ & $80.02$ & $81.97$ & $86.91$ &
                    $\mathbf{73.96}$ & $85.22$ & $89.25$ & $91.94$ & $93.43$ & $96.94$ \\
\MethodSvmTopK{2} & $63.11$ & $76.16$ & $79.02$ & $81.01$ & $82.75$ & $87.65$ &
                    $73.06$ & $85.67$ & $90.37$ & $92.24$ & $94.48$ & $97.31$ \\
\MethodSvmTopK{3} & $\mathbf{63.37}$ & $76.72$ & $79.67$ & $81.49$ & $83.57$ & $88.25$ &
                    $71.57$ & $\mathbf{86.27}$ & $\mathbf{91.12}$ & $93.21$ & $94.70$ & $97.24$ \\
\MethodSvmTopK{4} & $63.20$ & $76.64$ & $79.76$ & $82.36$ & $84.05$ & $88.64$ &
                    $71.42$ & $85.67$ & $90.75$ & $\mathbf{93.28}$ & $\mathbf{94.78}$ & $97.84$ \\
\MethodSvmTopK{5} & $63.29$ & $76.81$ & $80.02$ & $82.75$ & $84.31$ & $88.69$ &
                    $70.67$ & $85.75$ & $90.37$ & $93.21$ & $94.70$ & $\mathbf{97.91}$ \\
\MethodSvmTopK{10} & $62.98$ & $\mathbf{77.33}$ & $80.49$ & $82.66$ & $84.57$ & $89.55$ &
                    $70.00$ & $85.45$ & $90.00$ & $93.13$ & $94.63$ & $97.76$ \\
\MethodSvmTopK{20} & $59.21$ & $75.64$ & $\mathbf{80.88}$ & $\mathbf{83.49}$ & $\mathbf{85.39}$ & $\mathbf{90.33}$ &
                    $65.90$ & $84.10$ & $89.93$ & $92.69$ & $94.25$ & $97.54$ \\
\fi
\midrule
\iflongversion
\MethodUsuTopK{1} & $62.81$ & $74.60$ & $77.76$ & $80.02$ & $81.97$ & $86.91$ &
                    $73.96$ & $85.22$ & $89.25$ & $91.94$ & $93.43$ & $96.94$ \\
\MethodUsuTopK{2} & $63.55$ & $76.25$ & $79.28$ & $81.14$ & $82.62$ & $87.91$ &
                    $\mathbf{74.03}$ & $85.90$ & $89.78$ & $92.24$ & $94.10$ & $97.31$ \\
\MethodUsuTopK{3} & $63.94$ & $76.64$ & $79.71$ & $81.36$ & $83.44$ & $87.99$ &
                    $72.99$ & $\mathbf{86.34}$ & $90.60$ & $92.76$ & $94.40$ & $97.24$ \\
\MethodUsuTopK{4} & $63.94$ & $76.85$ & $80.15$ & $82.01$ & $83.53$ & $88.73$ &
                    $73.06$ & $86.19$ & $\mathbf{90.82}$ & $92.69$ & $\mathbf{94.48}$ & $97.69$ \\
\MethodUsuTopK{5} & $63.59$ & $77.03$ & $80.36$ & $82.57$ & $84.18$ & $89.03$ &
                    $72.61$ & $85.60$ & $90.75$ & $92.99$ & $\mathbf{94.48}$ & $97.61$ \\
\MethodUsuTopK{10} & $\mathbf{64.02}$ & $77.11$ & $80.49$ & $83.01$ & $84.87$ & $89.42$ &
                    $71.87$ & $85.30$ & $90.45$ & $\mathbf{93.36}$ & $94.40$ & $\mathbf{97.76}$ \\
\MethodUsuTopK{20} & $63.37$ & $\mathbf{77.24}$ & $\mathbf{81.06}$ & $\mathbf{83.31}$ & $\mathbf{85.18}$ & $\mathbf{90.03}$ &
                    $71.94$ & $85.30$ & $90.07$ & $92.46$ & $94.33$ & $97.39$ \\
\fi
\bottomrule
\end{tabular}
\iflongversion
\caption{Top-$k$ accuracy (\%). %
{\bfseries Top section:} State of the art. %
{\bfseries Middle section:} Baseline methods.
\MethodSvmPrecKMulti{k} and \MethodSvmRecKMulti{k}
are \MethodSvmPerf\ \cite{joachims2005support};
\MethodWsabie{m}{Q} is $\rm Wsabie^{++}$ \cite{Gupta2014}
with an embedding dimension $m$ and the queue size $Q$;
in the first part, $m=101$ for Caltech and $m=67$ for Indoor.
{\bfseries Bottom section:}
Top-$k$ SVMs:
\MethodSvmTopK{k} -- with the loss (\ref{eq:topk-loss});
\MethodUsuTopK{k} -- with the loss (\ref{eq:topk-usu}).
}
\label{tbl:small}
\end{table}\vspace*{-.5em}
\begin{table}[ht]\scriptsize\centering\setlength{\tabcolsep}{.4em}
\fi
\begin{tabular}{l|cccccc}
\multicolumn{7}{c}{\small\hspace*{5em}\textbf{SUN 397} (10 splits)} \\\toprule
\multirow{2}{*}{Top-1 accuracy}
&
XHE \cite{xiao2010sun} & $38.0$ &
LSH \cite{lapin2014scalable} & $49.48 \pm 0.3$ &
ZLX \cite{zhou2014learning} & $54.32 \pm 0.1$ \\
&
SPM \cite{sanchez2013image} & $47.2 \pm 0.2$ &
GWG \cite{gong2014multi} & $51.98$ &
KL \cite{koskela2014convolutional} & $54.65 \pm 0.2$ \\
\midrule
\midrule
Method & Top-1 & Top-2 & Top-3 & Top-4 & Top-5 & Top-10 \\
\midrule
\midrule
\MethodSvmOva & $55.23 \pm 0.6$ & $66.23 \pm 0.6$ & $70.81 \pm 0.4$ & $73.30 \pm 0.2$ & $74.93 \pm 0.2$ & $79.00 \pm 0.3$ \\
\MethodTopPushOva & $53.53 \pm 0.3$ & $65.39 \pm 0.3$ & $71.46 \pm 0.2$ & $75.25 \pm 0.1$ & $77.95 \pm 0.2$ & $85.15 \pm 0.3$ \\
\midrule
\MethodSvmRecKOva{1} & $52.95 \pm 0.2$ & $65.49 \pm 0.2$ & $71.86 \pm 0.2$ & $75.88 \pm 0.2$ & $78.72 \pm 0.2$ & $86.03 \pm 0.2$ \\
\iflongversion
\MethodSvmRecKOva{2} & $52.80 \pm 0.2$ & $64.18 \pm 0.2$ & $68.81 \pm 0.2$ & $71.42 \pm 0.2$ & $73.17 \pm 0.2$ & $77.69 \pm 0.3$ \\
\MethodSvmRecKOva{3} & $40.50 \pm 0.3$ & $56.01 \pm 0.2$ & $64.96 \pm 0.2$ & $70.95 \pm 0.2$ & $75.26 \pm 0.2$ & $86.32 \pm 0.2$ \\
\MethodSvmRecKOva{4} & $46.59 \pm 0.4$ & $59.87 \pm 0.6$ & $66.77 \pm 0.5$ & $70.95 \pm 0.4$ & $73.75 \pm 0.3$ & $79.86 \pm 0.2$ \\
\fi
\MethodSvmRecKOva{5} & $50.72 \pm 0.2$ & $64.74 \pm 0.3$ & $70.75 \pm 0.3$ & $74.02 \pm 0.3$ & $76.06 \pm 0.3$ & $80.66 \pm 0.2$ \\
\MethodSvmRecKOva{10} & $50.92 \pm 0.2$ & $64.94 \pm 0.2$ & $70.95 \pm 0.2$ & $74.14 \pm 0.2$ & $76.21 \pm 0.2$ & $80.68 \pm 0.2$ \\
\midrule
\midrule
\MethodSvmTopK{1} & $58.16 \pm 0.2$ & $71.66 \pm 0.2$ & $78.22 \pm 0.1$ & $82.29 \pm 0.2$ & $84.98 \pm 0.2$ & $91.48 \pm 0.2$ \\
\iflongversion
\MethodSvmTopK{2} & $58.81 \pm 0.2$ & $72.71 \pm 0.2$ & $79.33 \pm 0.2$ & $83.29 \pm 0.2$ & $85.94 \pm 0.2$ & $92.19 \pm 0.2$ \\
\MethodSvmTopK{3} & $\mathbf{58.97 \pm 0.1}$ & $73.19 \pm 0.2$ & $79.86 \pm 0.2$ & $83.83 \pm 0.2$ & $86.46 \pm 0.2$ & $92.57 \pm 0.2$ \\
\MethodSvmTopK{4} & $58.95 \pm 0.1$ & $73.54 \pm 0.2$ & $80.25 \pm 0.2$ & $84.20 \pm 0.2$ & $86.78 \pm 0.2$ & $92.82 \pm 0.2$ \\
\MethodSvmTopK{5} & $58.92 \pm 0.1$ & $\mathbf{73.66 \pm 0.2}$ & $80.46 \pm 0.2$ & $84.44 \pm 0.3$ & $87.03 \pm 0.2$ & $92.98 \pm 0.2$ \\
\MethodSvmTopK{10} & $58.00 \pm 0.2$ & $73.65 \pm 0.1$ & $\mathbf{80.80 \pm 0.1}$ & $\mathbf{84.81 \pm 0.2}$ & $\mathbf{87.45 \pm 0.2}$ & $93.40 \pm 0.2$ \\
\MethodSvmTopK{20} & $55.98 \pm 0.3$ & $72.51 \pm 0.2$ & $80.22 \pm 0.2$ & $84.54 \pm 0.2$ & $87.37 \pm 0.2$ & $\mathbf{93.62 \pm 0.2}$ \\
\fi
\midrule
\MethodUsuTopK{1} & $58.16 \pm 0.2$ & $71.66 \pm 0.2$ & $78.22 \pm 0.1$ & $82.29 \pm 0.2$ & $84.98 \pm 0.2$ & $91.48 \pm 0.2$ \\
\iflongversion
\MethodUsuTopK{2} & $58.80 \pm 0.2$ & $72.65 \pm 0.2$ & $79.26 \pm 0.2$ & $83.21 \pm 0.2$ & $85.85 \pm 0.2$ & $92.14 \pm 0.2$ \\
\MethodUsuTopK{3} & $59.14 \pm 0.2$ & $73.21 \pm 0.2$ & $79.81 \pm 0.2$ & $83.77 \pm 0.2$ & $86.36 \pm 0.2$ & $92.51 \pm 0.2$ \\
\MethodUsuTopK{4} & $59.24 \pm 0.1$ & $73.58 \pm 0.2$ & $80.18 \pm 0.2$ & $84.15 \pm 0.2$ & $86.71 \pm 0.2$ & $92.73 \pm 0.2$ \\
\MethodUsuTopK{5} & $59.28 \pm 0.2$ & $73.78 \pm 0.2$ & $80.45 \pm 0.3$ & $84.36 \pm 0.3$ & $86.96 \pm 0.3$ & $92.93 \pm 0.2$ \\
\fi
\MethodUsuTopK{10} & $\mathbf{59.32 \pm 0.1}$ & $\mathbf{74.13 \pm 0.2}$ & $80.91 \pm 0.2$ & $84.92 \pm 0.2$ & $87.49 \pm 0.2$ & $93.36 \pm 0.2$ \\
\MethodUsuTopK{20} & $58.65 \pm 0.2$ & $73.96 \pm 0.2$ & $\mathbf{80.95 \pm 0.2}$ & $\mathbf{85.05 \pm 0.2}$ & $\mathbf{87.70 \pm 0.2}$ & $\mathbf{93.64 \pm 0.2}$ \\
\bottomrule
\end{tabular}
\iflongversion
\\[2em]
\begin{tabular}{l|cccccc|cccccc}
\multicolumn{1}{c}{} &
\multicolumn{6}{c}{\textbf{\small Places 205} (val)} &
\multicolumn{6}{c}{\textbf{\small ImageNet 2012} (val)} \\\toprule
Method & Top-1 & Top-2 & Top-3 & Top-4 & Top-5 & Top-10
& Top-1 & Top-2 & Top-3 & Top-4 & Top-5 & Top-10 \\
\midrule
\midrule
{\tiny ZLX \cite{zhou2014learning} / BVLC \cite{jia2014caffe}} &
$50.0$ & - & - & - & $81.1$ & - &
$\mathbf{57.4}$ & - & - & - & $\mathbf{80.4}$ & - \\
\midrule
\midrule
\MethodTopPushOva & $38.45$ & $47.33$ & $53.25$ & $57.29$ & $60.30$ & $69.91$ &
$55.49$ & $68.05$ & $\mathbf{73.89}$ & $\mathbf{77.34}$ & $79.72$ & $\mathbf{85.99}$ \\
\midrule
\midrule
\MethodSvmTopK{1} & $50.63$ & $64.47$ & $71.44$ & $75.50$ & $78.54$ & $86.17$ &
$\mathbf{56.61}$ & $67.31$ & $72.43$ & $75.45$ & $77.67$ & $83.71$ \\
\MethodSvmTopK{2} & $51.05$ & $65.74$ & $73.10$ & $77.49$ & $80.74$ & $88.43$ &
$56.60$ & $68.09$ & $73.25$ & $76.36$ & $78.62$ & $84.55$ \\
\MethodSvmTopK{3} & $\mathbf{51.31}$ & $66.17$ & $73.23$ & $77.86$ & $81.26$ & $89.37$ &
$56.56$ & $68.27$ & $73.60$ & $76.76$ & $79.03$ & $84.96$ \\
\MethodSvmTopK{4} & $51.24$ & $\mathbf{66.30}$ & $73.48$ & $78.08$ & $81.40$ & $89.74$ &
$56.52$ & $68.36$ & $73.80$ & $77.06$ & $79.30$ & $85.25$ \\
\MethodSvmTopK{5} & $50.80$ & $66.23$ & $\mathbf{73.67}$ & $78.19$ & $81.43$ & $89.95$ &
$56.46$ & $\mathbf{68.40}$ & $\mathbf{73.85}$ & $77.20$ & $79.39$ & $85.41$ \\
\MethodSvmTopK{10} & $50.10$ & $65.76$ & $73.38$ & $\mathbf{78.30}$ & $\mathbf{81.62}$ & $\mathbf{90.14}$ &
$55.89$ & $68.16$ & $73.80$ & $\mathbf{77.31}$ & $\mathbf{79.75}$ & $85.77$ \\
\MethodSvmTopK{20} & $49.25$ & $64.85$ & $72.62$ & $77.67$ & $81.14$ & $89.99$ &
$54.94$ & $67.53$ & $73.50$ & $77.08$ & $79.59$ & $\mathbf{85.88}$ \\
\midrule
\MethodUsuTopK{1} & $50.63$ & $64.45$ & $71.45$ & $75.50$ & $78.54$ & $86.17$ &
$56.61$ & $67.31$ & $72.43$ & $75.45$ & $77.67$ & $83.71$ \\
\MethodUsuTopK{2} & $51.03$ & $65.58$ & $72.73$ & $77.40$ & $80.55$ & $88.40$ &
$56.91$ & $67.98$ & $73.19$ & $76.23$ & $78.50$ & $84.43$ \\
\MethodUsuTopK{3} & $51.27$ & $65.98$ & $73.37$ & $77.91$ & $81.25$ & $89.30$ &
$57.00$ & $68.27$ & $73.51$ & $76.68$ & $78.89$ & $84.84$ \\
\MethodUsuTopK{4} & $\mathbf{51.38}$ & $66.20$ & $73.56$ & $78.04$ & $81.40$ & $89.78$ &
$56.99$ & $68.39$ & $73.62$ & $76.86$ & $79.15$ & $85.09$ \\
\MethodUsuTopK{5} & $51.25$ & $\mathbf{66.25}$ & $\mathbf{73.66}$ & $78.26$ & $81.42$ & $89.91$ &
$\mathbf{57.09}$ & $\mathbf{68.45}$ & $73.68$ & $76.95$ & $79.27$ & $85.24$ \\
\MethodUsuTopK{10} & $50.94$ & $66.13$ & $73.52$ & $\mathbf{78.36}$ & $\mathbf{81.69}$ & $\mathbf{90.19}$ &
$56.90$ & $68.42$ & $\mathbf{73.95}$ & $77.31$ & $79.53$ & $85.62$ \\
\MethodUsuTopK{20} & $50.50$ & $65.79$ & $73.38$ & $78.17$ & $81.60$ & $90.12$ &
$56.48$ & $68.29$ & $73.83$ & $\mathbf{77.32}$ & $\mathbf{79.60}$ & $\mathbf{85.81}$ \\
\bottomrule
\end{tabular}
\fi
\caption{Top-$k$ accuracy (\%). %
{\bfseries Top section:} State of the art. %
{\bfseries Middle section:} Baseline methods.
{\bfseries Bottom section:}
Top-$k$ SVMs:
\MethodSvmTopK{k} -- with the loss (\ref{eq:topk-loss});
\MethodUsuTopK{k} -- with the loss (\ref{eq:topk-usu}).
\iflongversion
Results for Places 205 and ImageNet 2012 are computed on the validation set.
\vspace*{-1em}
\fi
}
\label{tbl:all}
\end{table}

%% file: figures/plot-scaling.tex
\begin{wrapfigure}[10]{r}{0.5\textwidth}%
\includegraphics[width=.95\linewidth]{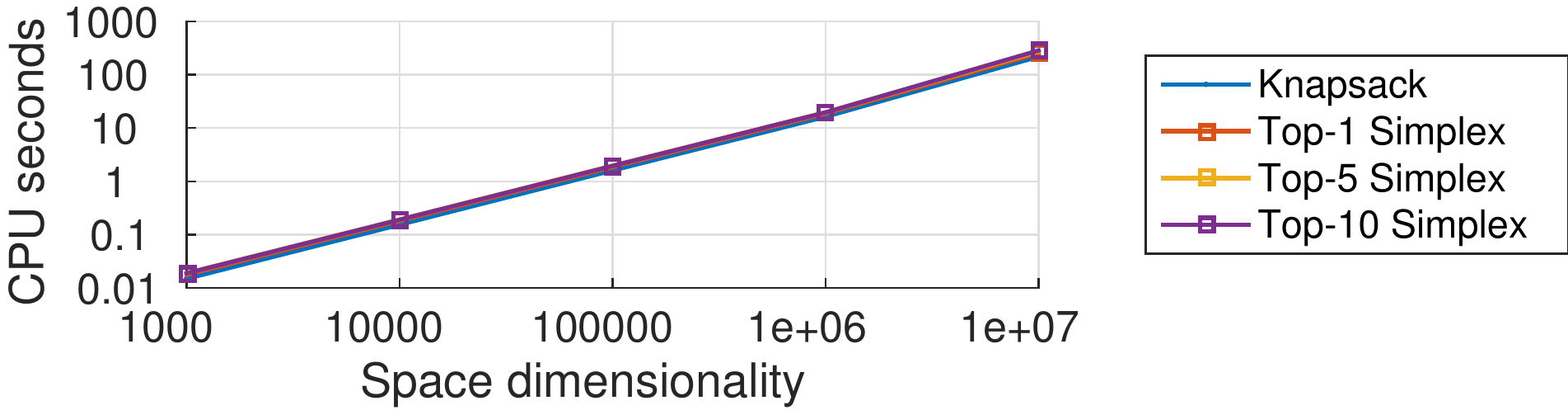}%
\caption{%
Scaling of the projection onto the top-$k$ simplex
compared to the knapsack problem.
}\label{fig:plot-scaling}%
\end{wrapfigure}

%% file: sections/conclusion.tex
We demonstrated scalability and effectiveness
of the proposed top-$k$ multiclass SVM on five image recognition datasets
leading to consistent improvements in top-$k$ performance.
In the future, one could study if the top-$k$ hinge loss (\ref{eq:topk-loss})
can be generalized to the family of ranking losses \cite{usunier2009ranking}.
Similar to the top-$k$ loss, this could lead to tighter convex upper bounds on
the corresponding discrete losses.